\documentclass[accepted]{article}

\usepackage{microtype}
\usepackage{graphicx}
\usepackage{subfigure}
\usepackage{booktabs} 

\usepackage{hyperref}

\usepackage{amsmath,amssymb,amsthm}

\newtheorem{theorem}{Theorem}

\newtheorem{proposition}[theorem]{Proposition}

  \theoremstyle{definition}

\newtheorem{example}{Example}

\usepackage{icml2020}

\icmltitlerunning{Clustering with Fast, Automated and Reproducible
assessment applied to longitudinal neural tracking}

\begin{document}

\twocolumn[
\icmltitle{C-FAR: a Clustering ensemble method with Fast, Automated and Reproducible
assessment applied to longitudinal neural tracking}




\begin{icmlauthorlist}
\icmlauthor{Hanlin Zhu}{rice}
\icmlauthor{Xue Li}{rice}
\icmlauthor{Liuyang Sun}{rice}
\icmlauthor{Fei He}{rice}
\icmlauthor{Zhengtuo Zhao}{rice}
\icmlauthor{Lan Luan}{rice}
\icmlauthor{Ngoc Mai Tran}{austin}
\icmlauthor{Chong Xie}{rice}
\end{icmlauthorlist}

\icmlaffiliation{rice}{Department of Electrical and Computer Engineering, Rice University , Houston TX 77251, USA}
\icmlaffiliation{austin}{Department of Mathematics, University of Texas at Austin, Austin TX 78712}

\icmlcorrespondingauthor{Ngoc M Tran}{ntran@math.utexas.edu}

\icmlkeywords{Machine Learning, ICML}

\vskip 0.3in
]




\begin{abstract}
Across many areas, from neural tracking to database entity resolution, manual assessment of clusters by human experts presents a bottleneck in rapid development of scalable and specialized clustering methods. To solve this problem we develop C-FAR, a novel method for Fast, Automated and Reproducible assessment of multiple hierarchical clustering algorithms simultaneously. Our algorithm takes any number of hierarchical clustering trees as input, then strategically queries pairs for human feedback, and outputs an optimal clustering among those nominated by these trees. While it is applicable to large dataset in any domain that utilizes pairwise comparisons for assessment, our flagship application is the cluster aggregation step in spike-sorting, the task of assigning waveforms (spikes) in recordings to neurons. On simulated data of 96 neurons under adverse conditions, including drifting and 25\% blackout, our algorithm produces near-perfect tracking relative to the ground truth. Our runtime scales linearly in the number of input trees, making it a competitive computational tool. These results indicate that C-FAR is highly suitable as a model selection and assessment tool in clustering tasks. 
\end{abstract}

\section{Introduction}
\label{introduction}

Over the past three decades, intensive research on developing novel clustering algorithms have resulted in a plethora of choices. However, in many applications, model selection is an arduous task performed by human experts: one needs to assess the quality of clusters through pairwise comparisons, then manually tune the algorithms' parameters. This is time consuming, subjective, and clearly do not scale. This presents a fundamental roadblock in efficient processing of massive datasets. This problem is common to text classification \cite{schutze2006performance}, entity resolution in databases \cite{wang2012crowder,vesdapunt2014crowdsourcing,gokhale2014corleone,mazumdar2017theoretical},
biomedical problems \cite{wiwie2015comparing} and a wide variety of crowd-sourcing tasks~\cite{mazumdar2017clustering}. 
The flagship example that ignited this project stems from spike sorting in neuroscience. This is the problem of assigning waveforms (spikes) to neurons, the first and fundamental step in processing neural data from electrode arrays. For \emph{in vivo} recordings, neurons can only be indirectly identified as a cluster of similar waveforms. Continuous improvements over electrode designs \cite{blanche2005polytrodes,chung2019high,hong2019novel} now enable almost continuous \emph{in vivo} recordings of over $100$ electrodes over a month, producing \emph{half a terabyte of data} per day per animal. Over such a large time span, the target neurons may undergo biological or positional changes over a day or a week \cite{rousche1992method,szarowski2003brain,subbaroyan2005finite,gilletti2006brain,barrese2013failure}, resulting in a different waveform. Existing spike-sorting algorithms \cite{takahashi2002classification,takekawa2012spike,carlson2013multichannel,rodriguez2014clustering,rossant2016spike,chung2017fully,song2018spike}
were developed under the assumption that a neuron's average waveform is constant, and thus are not effective for long, continuous recordings. In practice, automated spike sorting is performed on binned segments of 30 to 60 minutes, then the units in different time bins are aggregrated using clustering algorithms such as mutual nearest neighbors (mNNs) between adjacent bins
 \cite{chung2019high}. However, a neuron may not be active at all hours due to the subject's behavior, waveform drifting, or intermittent electrode failures. This results in many more waveform clusters than there could be neurons in the recording range (cf. Figure \ref{fig:result}). Though different clustering metrics and algorithms could be applied, after their applications, experts still need to spend hours to curate the results. Effectively, they are doing a manual assessment and selection of different clustering models' outputs using pairwise comparisons. 

\section{Our contributions} This project develops and implements C-FAR, a novel method for \textbf{fast, automated and reproducible} assessment of multiple hierarchical clustering algorithms simultaneously. Our algorithm takes any number of hierarchical clustering trees as input, then strategically queries pairs for human feedback, and outputs a provably optimal partition of the data among those nominated by the various input trees, along with a summary of the trees' contributions and deviations. Our approach builds on the recent work in \cite{gentile2019flattening}, which partitions the data by trimming a hierarchical clustering tree $T$ using pairwise comparisons. We generalize the binary search of \cite{gentile2019flattening} to an efficient and recursive algorithm to find the optimal partition across multiple trees. The core idea is illustrated in Figure \ref{fig:example.5}. 

Rather than being ``yet another clustering algorithm", C-FAR can be deployed to assess and improve any collection of hierarchical clustering methods. Unlike existing cluster assessment techniques \cite{seo2002interactively,nam2007clustersculptor,schreck2009visual,lex2010comparative,cavallo2018clustrophile} which mainly focus on visualizing different algorithms' outputs but otherwise let the user freely explore or reassign clusters, our algorithm requires the user to answer a set of queries strategically chosen by the algorithm, and thus is completely reproducible. 

Using MEArec \cite{Buccino_mearec}, we simulated a dataset of 96 neurons of two different types, with biophysically plausible drifts and waveforms. For each combination of parameters (number of trees and dropout rate), we simulated 100 trials and chose $m \in \{1,2,4,8,32,120\}$ random trees out of a set of 120 hierarchical clustering trees obtained by varying different parameters such as distance metric and linkage type (see Section \ref{sec:experiment} for details). Figure \ref{fig:cluster.runtime} shows the performance statistics of C-FAR averaged over 100 such trials. The C-FAR algorithm performs well even under adverse conditions and the runtime scales linearly in the number of input trees, making it a competitive computational tool.

\begin{figure}[htb]
\vskip 0.2in
\begin{center}
\centerline{\includegraphics[width=\columnwidth]{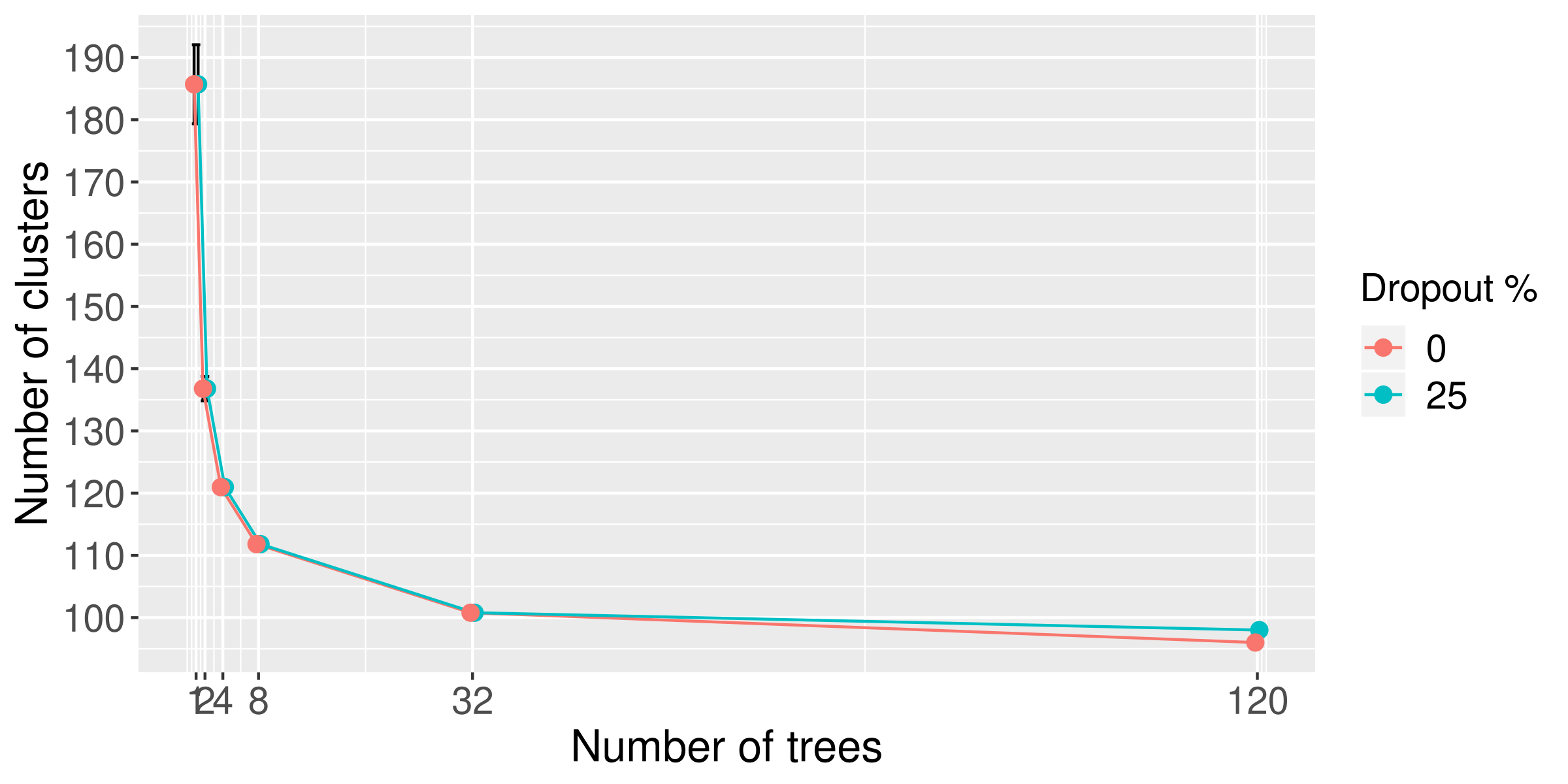}}
\centerline{\includegraphics[width=\columnwidth]{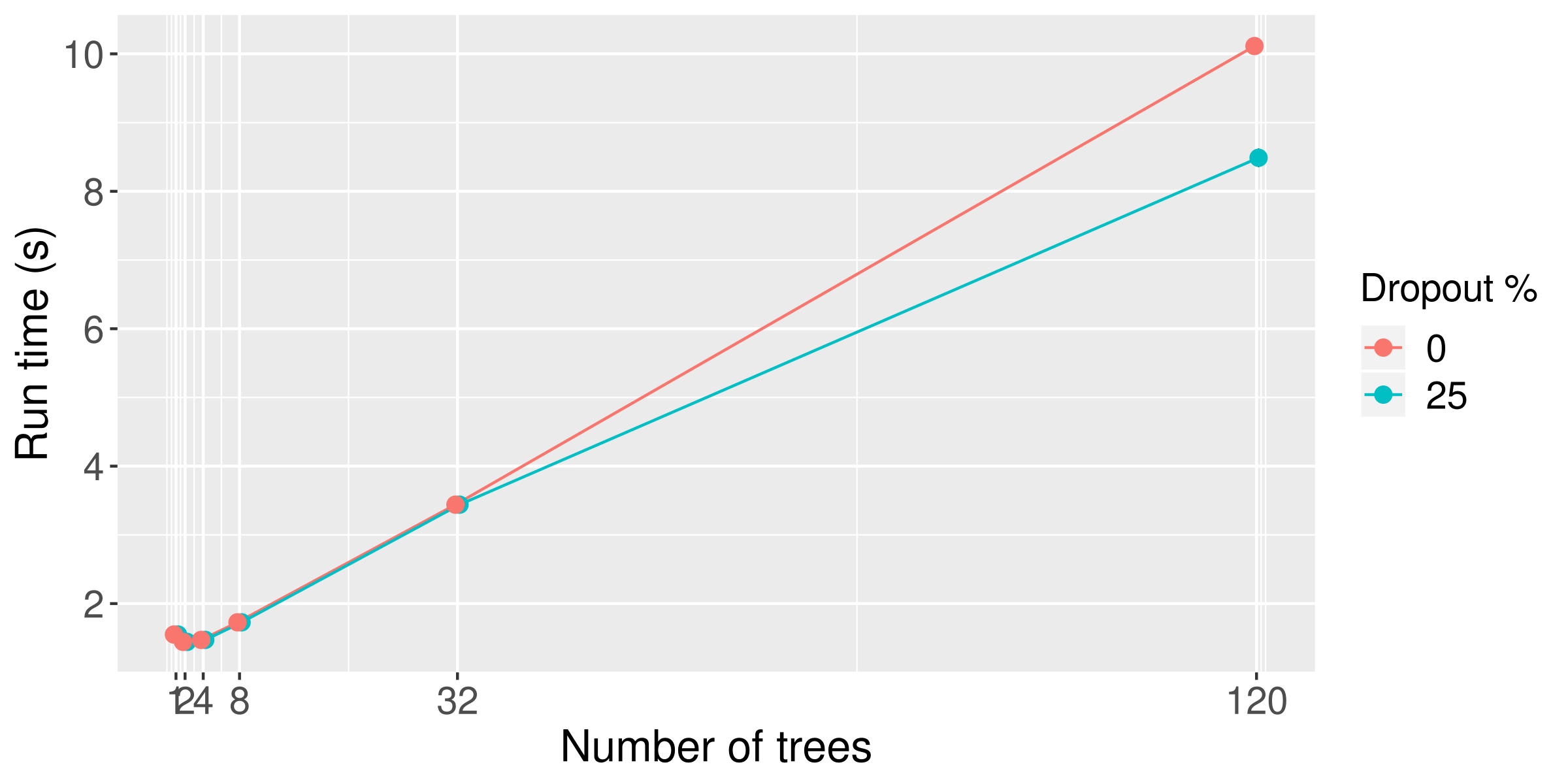}}
\caption{The C-FAR algorithm performs well even under adverse conditions and the runtime scales linearly in the number of input trees, making it a competitive computational tool. \textbf{Top.} Mean number of clusters by C-FAR vs number of input trees under two different experimental conditions: no dropout (red) vs 25\% dropout rate, meaning that each waveform has a 25\% chance of not registering in each timebin (green). The true number of clusters is 96. \textbf{Bottom.} Average runtime of C-FAR in seconds vs the number of input trees. In both plots, standard error bars are shown in black. The algorithm shows stable performances even with high dropout rates and near-perfect recovery of the clusters for 32 random trees. 
}
\label{fig:cluster.runtime}
\end{center}
\vskip -0.2in
\end{figure}

\begin{figure}[ht]
\vskip 0.2in
\begin{center}
\centerline{\includegraphics[width=0.32\columnwidth]{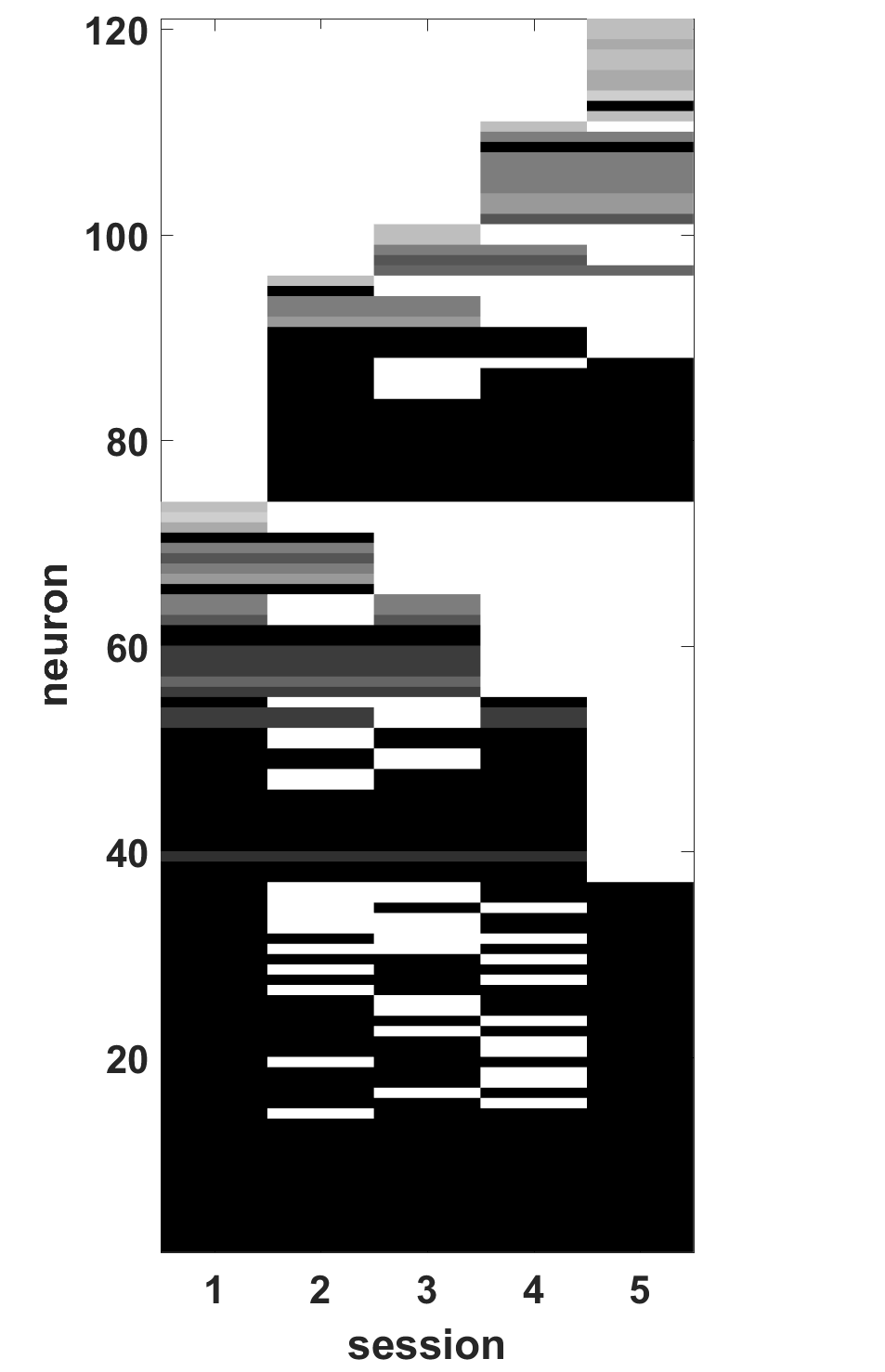}\, \includegraphics[width=0.32\columnwidth]{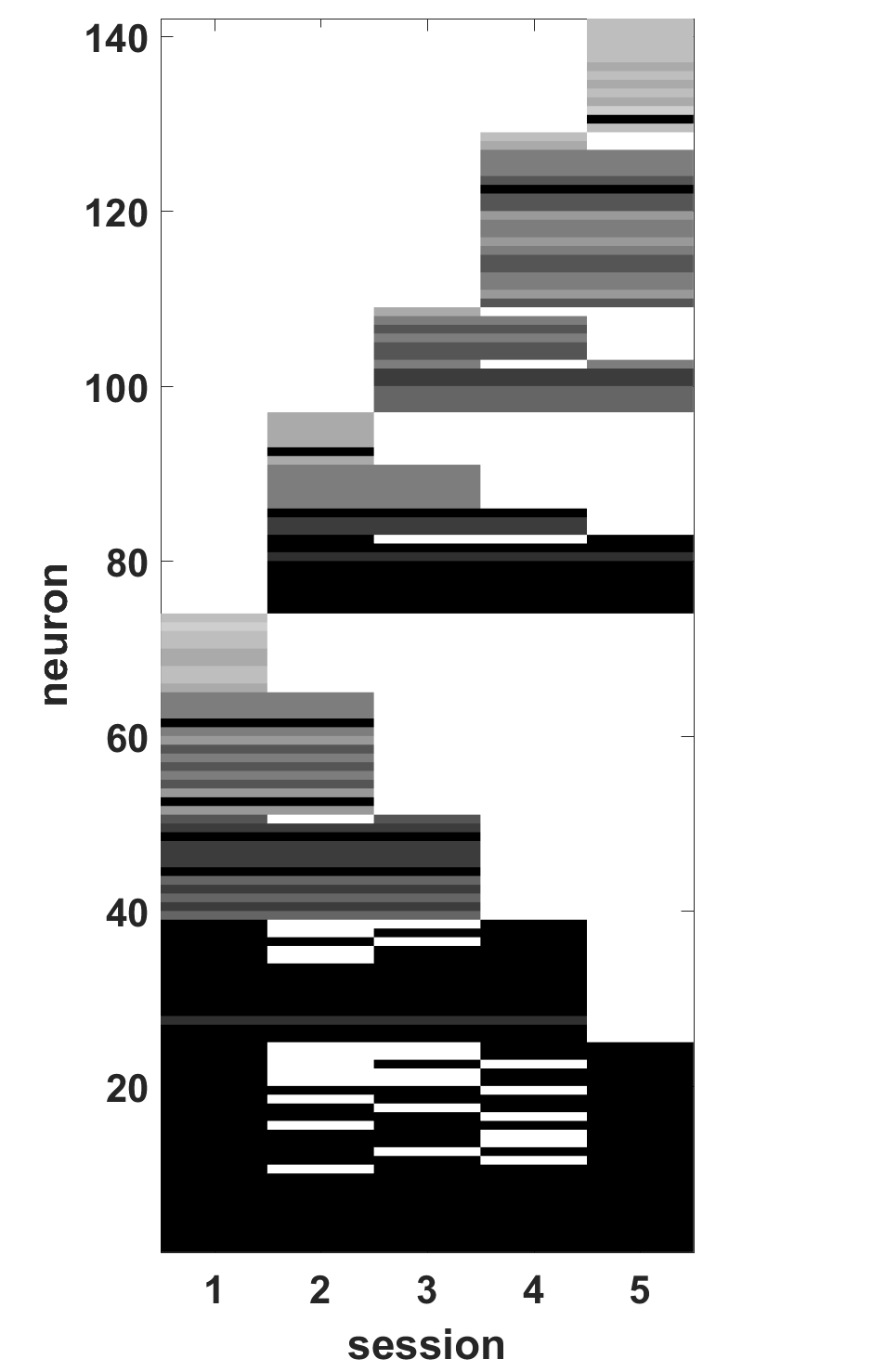}\, \includegraphics[width=0.32\columnwidth]{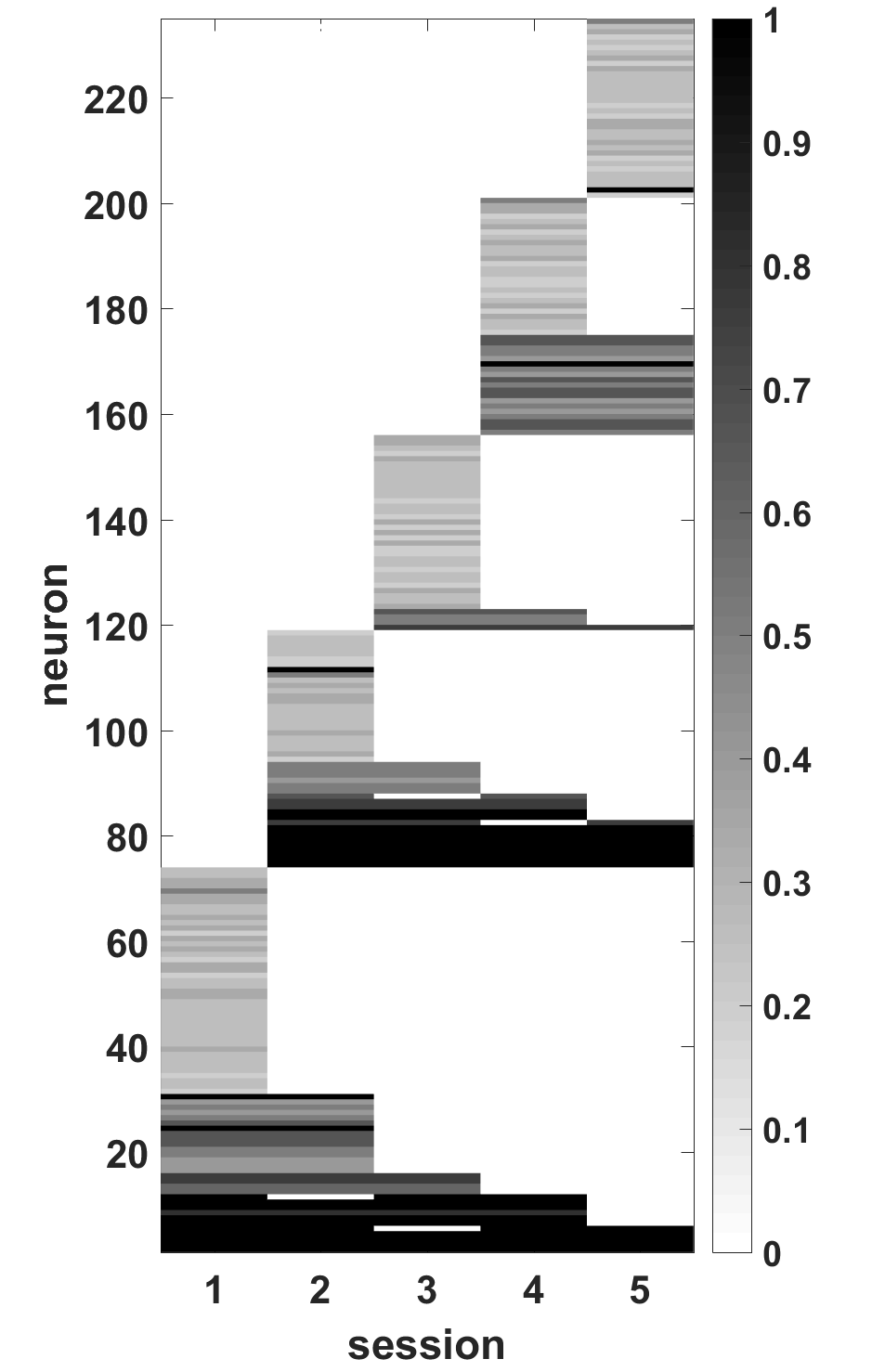}} 
\caption{On simulated neuroscience data with biophysically plausible parameters, C-FAR (left) is better than both the single-tree algorithm of \cite{gentile2019flattening} (middle) and the automated flattening comparable to clustering methods currently employed in the neuroscience community \cite{chung2019high} (right). The dataset has 96 neurons, simulated for 5 sessions. Each neuron for each session has a 25\% dropout chance. C-FAR was applied with 8 randomly chosen hierarchical clustering trees as input out of 120 trees built with different parameter choices (cf. Section \ref{sec:experiment}). The algorithm of \cite{gentile2019flattening} was applied to a random chosen tree out of these 120. Automated flattening (with no pairwise comparisons) was computed on this same tree such that each cluster has an average correlation of at least 0.96 between waveforms. For each plot, each row is a cluster corresponding to a putative neuron recovered from the respective method. The x-axis marks the sessions where the waveforms in this cluster appeared. The grayscale color of each cluster shows the recovery rate, defined as the number of sessions of this cluster divided by the number of sessions that the true neuron corresponding to the this cluster appeared in the dataset. Darker color indicates higher recovery rate. Black is 100\%, meaning all waveforms of this neuron are perfectly clustered. C-FAR produced 120 clusters, with 78\% neurons having 100\% recovery rate. The other two methods produced rather more fragmented clusters: \cite{gentile2019flattening} produced 140 units while automated flattenig produced 224 units, and their perfect recovery rates are lower, at 55\% and 26\%, respectively. The distribution of recovery rates for each method is shown in the opposite column.}
\label{fig:result}
\end{center}
\vskip -0.2in
\end{figure}

\begin{figure}[h]
\vskip 0.2in
\begin{center}
\centerline{\includegraphics[width=\columnwidth]{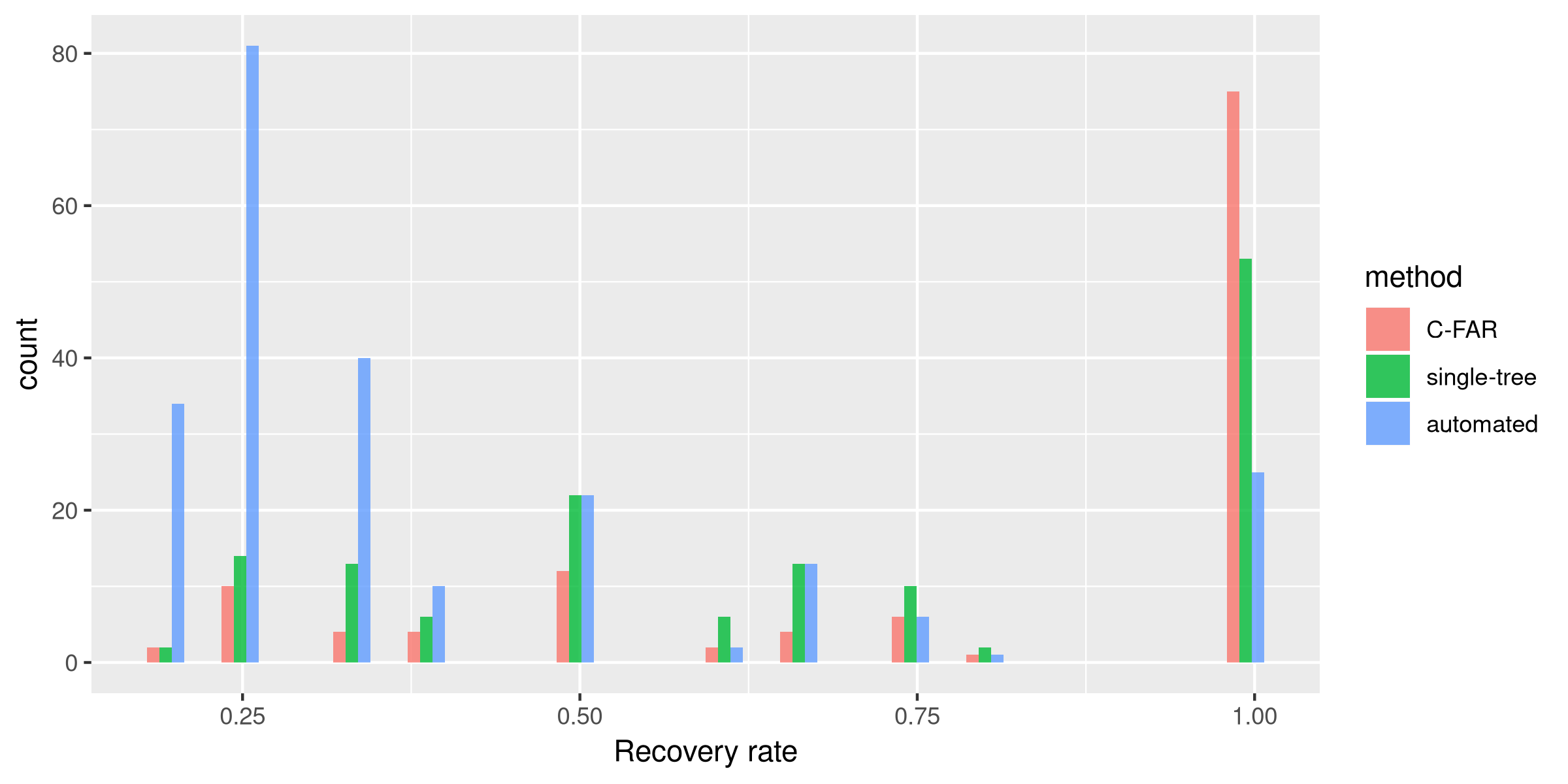}}
\caption{Histogram of recovery rate, accompanies Figure \ref{fig:result}.}\label{fig:recovery.rate}
\end{center}
\vskip -0.2in
\end{figure}

%


\subsection{Contributions to the spike-sorting literature}
Most spike sorting algorithms start with a dimension reduction step such as principal component analysis (PCA), followed by a clustering step. There is a large methodology literature that builds on a combination of many techniques: independent component analysis \cite{takahashi2002classification}, Kalman filter \cite{calabrese2011kalman}, variational Bayes \cite{takekawa2012spike}, mixture modeling \cite{carlson2013multichannel}, peak-density \cite{rodriguez2014clustering}, singular value decomposition \cite{pachitariu2016fast}, template matching \cite{yger2016fast}, random projections \cite{chung2017fully}, integer programming \cite{dhawale2017automated}, convolutional dictionary learning \cite{song2018spike} and deep neural networks \cite{lisupervised}, to name a few. A number of these methods are parametric, and thus can be generalized to allow time-varying features. 

However, given the plethora of choices, the bottleneck in fact lies in assessing their results. Even over the short time frame of several hours, automated spike sorting is very difficult due to a complex noise distribution, non-Gaussian clusters and the presence of biological irregularities such as bursting \cite{harris2001temporal,quirk2001experience}. All but a few of the above methods still require hours of manual curation \cite{chung2017fully}. For longitudinal sorting, the dependency on human experts is only amplified. Direct observation of a neuron is costly, requires special equipment, and can only monitor one neuron at a time \cite{yger2018spike}. In practice, clusters are still assessed by experts through pairwise comparisons, with \emph{ad hoc} reassignments and minimal justifications. Instead of yet-another-algorithm, C-FAR supplies the neuroscience community with a fast and reproducible way to assess the quality of multiple clustering methods.

\subsection{Contributions to the clustering literature}
Our proposed C-FAR algorithm functions as a fast algorithm for clustering, as well as a diagnostic tool for comparisons of different hierarchical clustering outputs and an automated model selection method. It straddles across multiple literatures; each offers solutions to one of the above problems but not all at once. 

A number of papers have considered clustering algorithms that minimize number of pair comparisons. Unfortunately, existing algorithms either require unrealistic assumptions such as noise-free comparisons and tight clusters \cite{eriksson2011active}, are specific to a particular clustering algorithm \cite{shamir2011spectral,wauthier2012active,ailon2017approximate,ashtiani2016clustering,chatziafratis2018hierarchical}, or are slow, too general and do not take advantage of existing information offered by the input clustering methods \cite{dasarathy2015s2,mazumdar2017clustering,chen2014clustering,krishnamurthy2012efficient}. On the other extreme, ensemble clustering methods aggregrate multiple clustering outputs to compute a new partition \cite{vega2011survey,alqurashi2019clustering,fred2005combining,strehl2002cluster,karypis1998fast,fern2004solving,iam2011link,mimaroglu2011diclens,huang2015robust,ayad2010voting,dudoit2003bagging,fern2003random,gionis2007clustering}.
There are many alternative approaches; among them, object co-occurence uses pairwise comparisons induced by the different clusters to produce a similarity matrix from which the final partition is computed \cite{monti2003consensus,strehl2002cluster,fred2005combining}.  
Most methods do not take advantage of cluster similarity measures coming from a hierarchical clustering and thus are often slow, with quadratic complexity in the number of clusters. Critically, unlike our proposed method, user feedback is not taken into account in the construction of the similarity matrix, thus a separate model selection step is required as with any other clustering algorithms. 

The difficulty of comparing different clustering methods is common to many applications, most notably in genetics. A variety of methods have been developed to visually compare multiple clustering algorithms \cite{cao2011dicon,lex2012stratomex,l2015xclusim,pilhofer2012comparing,zhou2009visually,sacha2017somflow}. A number of these are interactive and allow the user to reassign clusters, supervise a chosen clustering algorithm, or give weights to selected observations \cite{seo2002interactively,nam2007clustersculptor,schreck2009visual,lex2010comparative,cavallo2018clustrophile,balcan2008clustering,awasthi2017local}. 
However, the sheer volume of data makes these methods time-consuming to use for neural tracking. More importantly, they often lack reproducibility since users have too much freedom in selecting the type of feedback to give, and not all of the user's decisions are systematically recorded. Furthermore, analyzing the feedback can itself be a complex task, since it can be unclear which decisions are the most consequential in shaping the final partition. 

\begin{figure}[tb]
\includegraphics[width=\columnwidth]{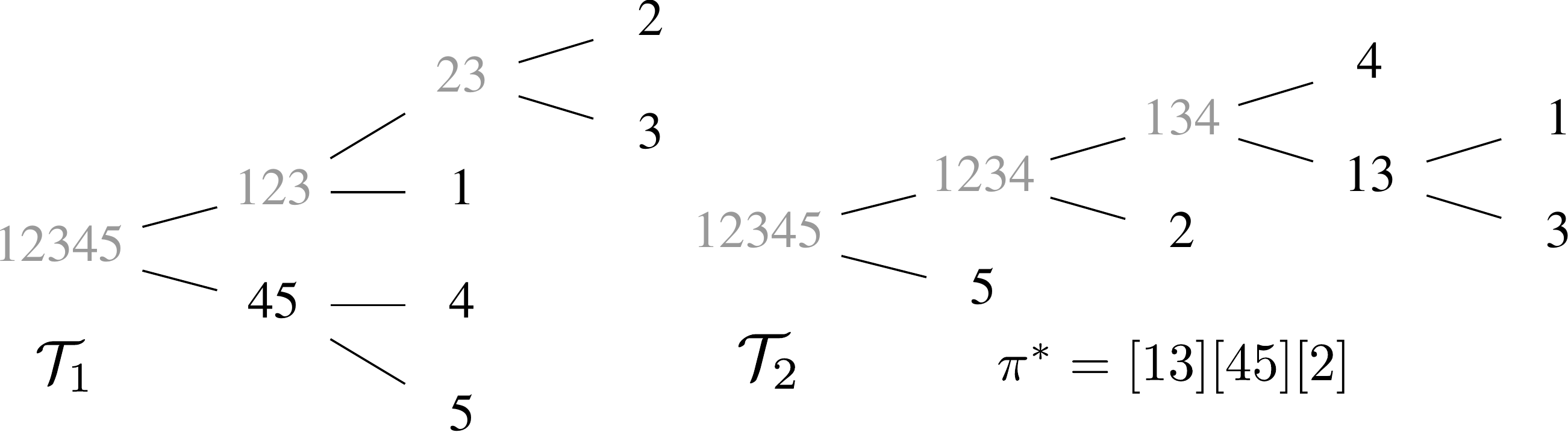}
\caption{Core idea of C-FAR on a simple example with $n = 5$ observations. The true partition~$\pi^\ast$ induces a purity process on each hierarchical clustering tree, where a node is pure (black) if observations in the corresponding partition come from the same true cluster. The purity of a given node can be determined from pairwise comparisons between the leaves of its left and right children. Purity satisfies important inequalities imposed by set inclusion, making it monotone along the branches and allow information to be pooled across trees. For example, if the queries show that $y(23) = 0$ in tree $\mathcal{T}_1$, then $y(1234) = 0$ in tree $\mathcal{T}_2$, and all of their ancestors are also not pure. 
Our algorithm leverages this fact to search over all trees at once using the same set of pairwise comparisons. It is a fast and recursive way to compute $\pi^\ast$, with provable accuracy and complexity.} \label{fig:example.5} 
\vskip-0.8cm
\end{figure}

\subsection{Relation to clustering with active learning}
Assessing clusters' qualities by pairwise comparisons is close to active learning, where a learner is given access to unlabeled data and is allowed to adaptively choose which ones to label \cite{lewis1994sequential}. Both share the difficulty of finding a representative sample \cite{zhao2003representative}, however, cluster assessment has the added difficulty that our search space is over ${n \choose 2}$ pairs instead of $n$ observations \cite{xiong2016active}. Uniform sampling is generally inefficient \cite{eriksson2011active}. 

To see this, suppose that data consist of $k$ clusters generated by independently assigning the label $1$ to $k$ to each observation. Suppose that a partition $\pi$ that is incorrect on $p$ fraction in each cluster. Now consider a randomly chosen pair of observations $(u,v)$. The probability that they belong to different true clusters and are wrongly put together by $\pi$ is $O(p/k)$. The probability that they belong to the same true cluster and are wrongly separated by $\pi$ is also $O(p/k)$. Thus overall, the probability that $\pi$ makes a mistake on a randomly chosen pair $(u,v)$ is $O(p/k)$.  In particular, if the number of clusters $k$ is large, as often the case in applications, then uniform sampling is a bad strategy to uncover misclustered pairs.

One may be tempted to sample pairs closer to the boundary of adjacent clusters. For $k=2$, this strategy was the default in the early days of active learning \cite{lewis1994sequential,campbell2000query,tong2001support,zhao2003representative}, but was later shown to induce a serious sampling bias \cite{dasgupta2008hierarchical}. Ideally, one would like to sample within clusters to measure purity, sample between clusters to measure modularity, then sample on the boundary for misclassification. A number of papers on active clustering are built on this intuition \cite{huang2007semi,mallapragada2008active,wang2011approximate,hofmann1998active}.
Tractable models for theoretical analysis often assume a specific clustering model and advocate for selecting the most `informative' pair, or one that minimizes empirical uncertainty \cite{nguyen2004active, dasgupta2008hierarchical}.
Some papers view the pairwise comparisons as must-link or must-not-link constraints, then adapt each existing clustering algorithm such as $k$-means or spectral clustering into one that would respect these constraints \cite{xiong2012online,voiron2016deep,xiong2016active,lipor2017leveraging}. These approaches produce samples highly adapted to a single clustering method, which may not be suitable for assessing another. In other words, they do not advert the key difficulty: how to produce a representative sample that allows fair comparisons between methods? 

\section{Main results}\label{sec:innovation}

Our work generalizes the recent approach by \cite{gentile2019flattening}, who used active learning with pairwise comparisons to create partitions from one hierarchical clustering tree. We introduce the concept of node purity, and note that this satisfies important inequalities imposed by the lattice of subsets of $[n]$ ordered by inclusion. This allows us to generalize the binary search idea that drives the algorithms of \cite{gentile2019flattening}. As a result, our algorithm can take multiple trees as input, and utilize the same set of pairwise comparisons to simultaneously trim all the trees at once to arrive at the optimal partition. We expect our algorithm to have the same fast performance as that theoretically guaranteed in \cite{gentile2019flattening}, and this is supported by experiments (cf. Figure \ref{fig:cluster.runtime}). First we briefly review the results of \cite{gentile2019flattening} before stating our general setup and algorithms. 

\subsection{Partition from one tree by pairwise comparisons}\label{subsec:one.tree}
Let $T = (V,E)$ be a binary tree obtained from applying a hierarchical clustering procedure on some dataset with $n$ observations. A feedback matrix $\Sigma \in \{\pm 1\}^{n \times n}$ is a symmetric matrix encoding the binary relations between the observations: $\Sigma_{ij} = 1$ if observations $i$ and $j$ are similar, and $\Sigma_{ij} = -1$ if they are not. Given $T$, an active learning algorithm proceeds in a sequence of rounds. At round $t$, the algorithm chooses a pair $(i,j) \in [n] \times [n]$, and observes the associated label $\Sigma_{ij}$. At some point, the algorithm is stopped, and is compelled to produce a partition of $[n]$ by selecting a collection of nodes in $V$ that corresponds to mutually disjoint subsets. The goal is to return a good partition from $T$ by making as few queries on $\Sigma$ as possible.
We shall assume that there is a ground-truth feedback matrix $\Sigma^\ast \in \{\pm 1\}^{n \times n}$ induced by some true partition $\pi^\ast$ of $[n]$, where $\Sigma^\ast_{ij} = 1$ if and only if $i$ and $j$ belongs to the same block of $\pi^\ast$. In the noise-free case, $\Sigma = \Sigma^\ast$. In the noisy case, $\Sigma$ is obtained from $\Sigma^\ast$ by switching the sign of $\lambda {n \choose 2}$ entries for some $\lambda \in (0,1)$. Say that $\pi^\ast$ is realizable by $T$ if each of its blocks corresponds to a node of $V$. 

The key observation of \cite{gentile2019flattening} is that if $\pi^\ast$ is realizable by $T$, then it can be recovered by performing a binary search on paths that connect the root to the leaves.
Indeed, let $y: 2^{[n]} \to \{0,1\}$ be the purity function defined on subsets of the data, where $y(S) = 1$ if all observations in $S$ belong to the same block of $\pi^\ast$, and $0$ else. We can define $y$ on the nodes of $T$ by identifying them with the subset of observations that they represent. Since the clusters are hierarchical, for any leaf $\ell \in T$, the unique path from the root to $\ell$ has monotone label: if $v$ is pure, then any child of it is pure; if $v$ is not pure, then any parent of it is also not pure (cf. Figure \ref{fig:example.5}). Now, purity of a node $v$ can be decided by pairwise comparisons of elements from its left and right children. In the noiseless case, one comparison will suffice. In the noisy case, one can do repeated sampling and take majority. Thus one can find the largest pure node along any a given path via binary search, and such a node must be a block of $\pi^\ast$ by the realizability assumption. The algorithms in \cite{gentile2019flattening} select paths with maximum entropy at each step, do a binary search to find a new block of $\pi^\ast$, update the tree and repeat accordingly. 

\subsection{Best partition from multiple trees with the same set of comparisons}
We generalize the above approach to multiple trees as follows. Suppose there are multiple hierarchical clustering methods, which give yield to multiple trees $\mathcal{T} = \{T_1, T_2, \dots, T_r\}$. Say that the setup is realizable if each block of $\pi^\ast$ correspond to a node in one of the trees in $\mathcal{T}$. Let $y$ be the purity process as above. That is, for a node $v_i$ in tree $T_i$, $y(v_i) = 1$ if the cluster corresponding to $v$ is a subset of a block of $\pi^\ast$, and $0$ else. Now, $y$ satisfies important inequalities imposed by the lattice of subsets of $[n]$ ordered by inclusion. Namely, for any two subsets $v_i,v_j$ of the data, 
\begin{align}
& y(v_i \cap v_j) \geq \max(y(v_i),y(v_j)) \nonumber, \\
\mbox{ and } & y(v_i \cup v_j) = \min(y(v_i), y(v_j)). \label{eqn:important.ineq}
\end{align}
These constraints generalize the key observation of \cite{gentile2019flattening} that purity is a monotone non-decreasing process along any branch of a hierarchical clustering tree. 
When there are multiple trees, \eqref{eqn:important.ineq} enables one to pool information on node purity across multiple trees. This suggests the following algorithm to find one block of the optimal partition. It generalizes the binary search of \cite{gentile2019flattening}.

\begin{algorithm}[tb]
\caption{\texttt{FindOneBlock}} \label{alg:one.block}
\flushleft 
\textbf{Input}: hierarchical clustering trees $\mathcal{T} = \{T_1, T_2, \dots, T_r\}$ on $[n]$. \\
\textbf{Output}: one node among the trees $\mathcal{T}$ which corresponds to one cluster best compatible with the pairwise feedback
\begin{algorithmic}[1]
\STATE Start with a set $S \subset [n]$, maximal w.r.t inclusion, which has different minimal extensions $v_1, v_2,\dots v_r$ in the $r$ trees. 
\STATE Determine the purity of $v_1, \dots, v_r$ by sampling from $v_1 \backslash S, \dots, v_r \backslash S$. 
\STATE Renumber so $\{v_1,\dots,v_m\}$ is the set of pure nodes
  \IF{$m = 1$}
    \STATE Do binary search up the tree $T_1$ on this branch until we find the largest node $B$ (by inclusion) with purity 1
  \ELSE
    \STATE Look for the minimal extensions $w_1,\dots,w_r$ of $S' := v_1 \cup \dots \dots v_m$ in each tree.
    \STATE Repeat line 2 with $(v_1,\dots,v_r) := (w_1,\dots,w_r)$ and $S := S'$.
  \ENDIF
\STATE \textbf{return } the node $B$
\end{algorithmic}
\end{algorithm}

\begin{example}
Let the two trees depicted in Figure \ref{fig:example.5} be input to Algorithm \ref{alg:one.block}. In line 1, we can choose $S = \{3\}$. Then $v_1 = \{2,3\}$ and $v_2 = \{1,3\}$.  Since $y(\{1,3\}) = 1$, we can conclude that $y(23) = 0$, so $m = 1$. Algorithm \ref{alg:one.block} will then do a binary search along the branch $12345 - 1234 - 134 - 13$, find that $B = \{1,13\}$, and output this block of $\pi^\ast$. 
\end{example}

\begin{proposition}
In the noiseless case, if $\pi^\ast$ is realizable, then Algorithm \ref{alg:one.block} outputs one block of $\pi^\ast$.
\end{proposition}
\begin{proof}[Proof sketch]
Since $\pi^\ast$ is realizable by $\mathcal{T}$, the initial set $S$ must have $y(S) = 1$ and at least one of the $v_1,\dots,v_r$ must be pure, by minimality, so $m \geq 1$. At each iteration, the algorithm produces another pure set $S'$, and thus the same argument applies. Repeating this argument shows that the algorithm terminates with finding one block of $\pi^\ast$. Induction on the number of remaining blocks of $\pi^\ast$ concludes the proof.
\end{proof}

\section{Experimental results}
\subsection{Data simulation and choice of parameters}\label{sec:experiment}
While the C-FAR algorithm could still perform well in more adverse settings, we have taken care to simulate datasets with biologically plausible parameters. In particular, drifting waveform templates of extracellularly recorded neurons were simulated using MEArec \cite{Buccino_mearec}. Modeling the probes reported in \cite{zhao2017nanoletter}, we simulated 32 channels recording electrodes arranged into a 2D array of shape $4 \times 8$ with an electrode center to center distance of 50 $\mu$m. Two types of rats somatosensory cortex neuron models included in the paper \cite{Buccino_mearec} \cite{ramaswamy2015neocortical} were selected : bitufted cells (BTC) and double bouquet cells (DBC). For each of the two neuron types, 48 neuron instances were initially randomly positioned within the area of $92,400 \mu$m$^2$ covered by the 2D electrode array. Similar to \cite{Hurwitz_NIPS2019}, an overhang detection range, in our case 35 $\mu$m, was allowed to account for the detectable neurons that reside/drift slightly beyond the electrode 2D boundary. Assuming the same overhang, one of the latest in-vivo high density flexible electrodes \cite{chung2019high} reported an average of $\sim$16 neurons (max 45) detected in an array area of $23 320$ $\mu$m$^2$. Assuming that neuron number directly scales with recording area, we would get an average of $63$ (max 178) neurons. Therefore, our chosen number of $96$ neurons is well within this range. The vertical distance (height) of neuron were bounded within 10 $\mu$m to 80 $\mu$m relative to the electrode plane. So far, unanimous opinion or evidence on how fast or in which directions neurons might drift or migrate with respect to the recording electrodes is not present. For some electrode arrays, no preferred moving directions could be observed \cite{luan2017ultraflexible} for a few weeks during which neurons moved $\sim$10 $\mu$m. For other electrode designs, neurons could move as far as 20 $\mu$m mostly along vertical directions \cite{Pachitariu2018drift} in a single experiment session. For our simulation, the 96 neurons were simulated to move with a given constant velocity for 5 steps resulted in a total of 480 detected units, analogous to 480 spike-sorted neurons detected across a total of 5 intermittent in-vivo recording sessions. With drifting magnitude of at least 7 $\mu$m per step, the number of possible drifting directions were limited to a maximum of 5 per type of neuron. To further increase the difficulty of the task, physiologically feasible random rotations were applied at each time step as described in \cite{Buccino_mearec}. Finally, in reality, the number of activated and detectable neurons at any given time step are affect by numerous factors such as the behavioral state of the animal \cite{blanche2005polytrodes}. To simulate random inactivation of neurons, in a separate dataset, a 25 \text{\%} drop out chance was applied to the 96 neurons at each of the 5 drifting steps, meaning that the waveform for this neuron is missing for this step. This resulted in a total of 360 detected units. Recorded waveforms on any channel was down-sampled to have 38 time points. Therefore, the waveform of each of 480 or 360 units was represented as a 3D array of shape $4 \times 8 \times 38$, which was later spatially flattened into $1 \times 1216$ array. A list of labels indicating which of the 96 seed neurons that the 480 or 360 units correspond to was also generated to serve as the ground truth. 
\subsection{Creating the original hierarchical clustering ensemble}
Agglomerative hierarchical cluster trees were created with the original MATLAB function \texttt{linkage()} based on the pairwise waveform distance of the input units. By varying parameters in the raw waveform processing step (2 choices\footnote{raw or take first derivative in time for each channel}), distance transformations (3 choices\footnote{none, first ten PCA components across units or first 3 components of tSNE across units}), distance metric (5 choices\footnote{Euclidean, squared Euclidean, Manhattan, Chebychev or correlation}) and linkage type (4 choices\footnote{single, average, complete, weighted}), we obtain 120 different trees, one for each combination of parameters. 

\subsection{Implementation of C-FAR}
Each of the original hierarchical clustering tree was first converted into a tree object instance of a customized MATLAB tree class adopted from  \cite{tinevez}. Each node of this first tree (named composition tree) stores a list of unit IDs clustered at this node. A second tree (named purity tree) whose structure is synchronized with the first tree, stores the purity value (whether the clustering defined by this node in the corresponding composition tree is correct, wrong, or unknown) in each node. \footnote{For speed improvements, the forest formed by all composition trees was further converted into a 3D boolean array of shape (number of nodes)-by-(number of units)-by-(number of trees) through one hot encoding of unit IDs}. Next, ground truth cluster labels was converted into a pairwise ground truth decision matrix indicting whether any given two pairs of unit belongs to the same cluster. After that, \texttt{FindOneBlock} was initialized with a leaf node (pure by definition) and executed through repeatedly asking users ( or in this case, the ground truth decision matrix ) the purity of some nodes. Before the next query, all past answers obstained were consulted first to see if the answer to the current question could already be inferred before consulting ground truth matrix or users. Whenever an impure node is identified at any step, the purity of all nodes inside the forest whose composition is a super set of that of the impure node is updated as impure simultaneously so their purity will not be queried for user feedback again. At the end of \texttt{FindOneBlock} when one block of true clusters was found. Units forming this block would then be appended to the clustering result, eliminated from all trees before \texttt{FindOneBlock} was executed for another iteration until all forest became empty. Between iterations, purity trees were reinitialized with to be the state unknown.
\subsection{Evaluation and results}\label{sec:results}
Dateset was generated on Linus OS (Ubuntu 18.04) using the command line interface provided \cite{Buccino_mearec}. Hierarchical clustering as well as C-FAR were performed with Matlab 2017a running on a Windows 10 computer equipped with 16-core Intel Xeon E5-2670 CPU. Converting from original hierarchical clustering result returned by MATLAB function \texttt{linkage()} into the C-FAR ready data structure was performed with 8-core parallel computation with using MATLAB parallel computing toolbox. Performance metrics for accuracy are the number of clusters (cf. Figure \ref{fig:cluster.runtime}, top) and adjusted mutual information (AMI) between true and inferred labels \cite{vinh2010JMLR,vinh2009icml} (cf. Figure \ref{fig:ami}). Figure \ref{fig:cluster.runtime} (bottom) shows the run time in seconds, while Figure \ref{fig:ami} (bottom) shows the average number of queries (number of pairwise comparisons) a user must provide for the algorithm to outputs a clustering based on the given trees. For one tree, \cite{gentile2019flattening} showed that the number of queries is $O(n)$ where $n$ is the number of leaves of the tree. Thus our experimental results show that the number of queries of C-FAR scales as $\log(n) + \log(m)$, where $m$ is the number of input trees and $n$ is the initial number of clusters  (cf. Figure \ref{fig:ami}).

\begin{figure}[htb]
\vskip 0.2in
\begin{center}
\centerline{\includegraphics[width=\columnwidth]{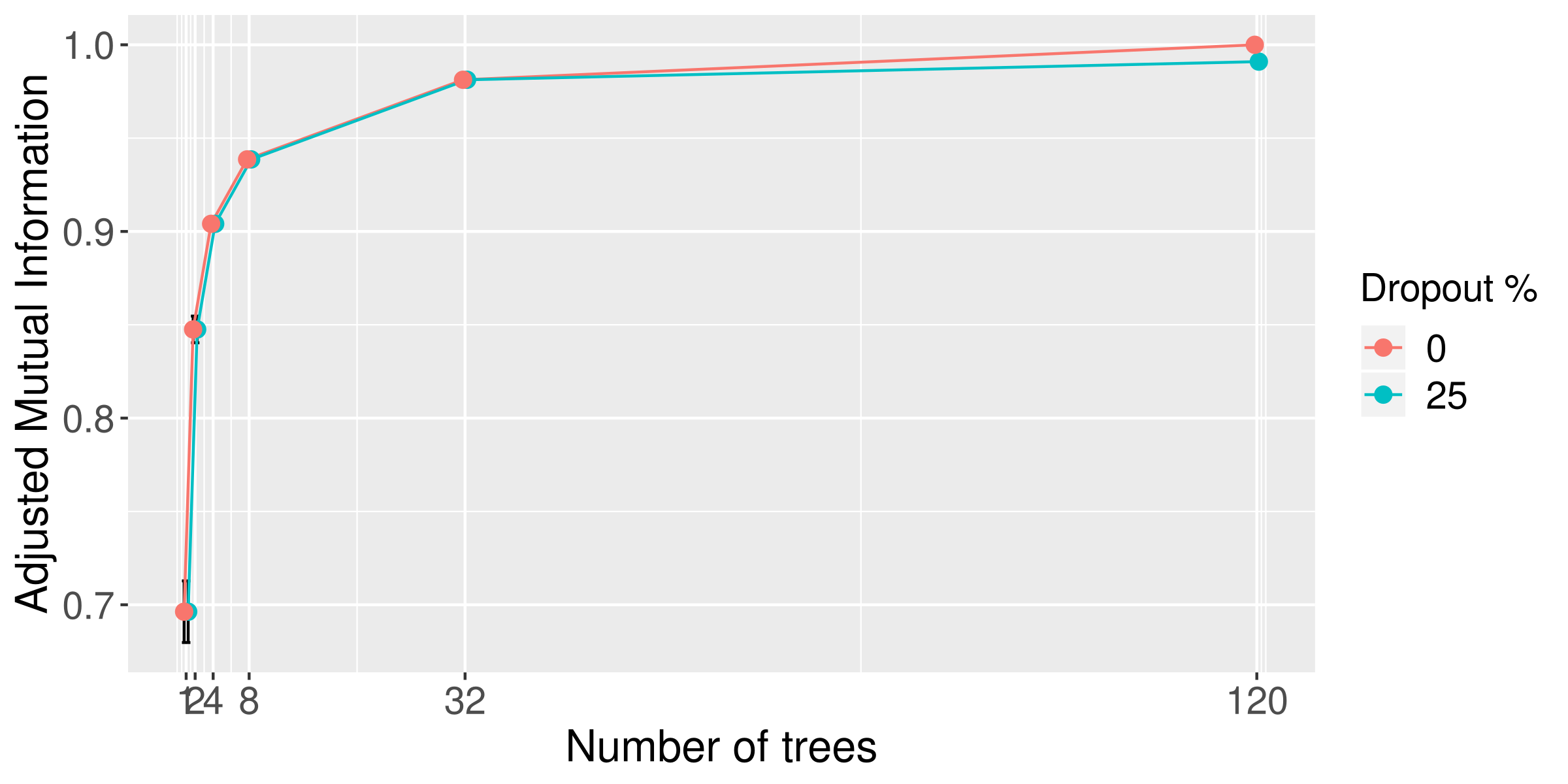}}
\centerline{\includegraphics[width=\columnwidth]{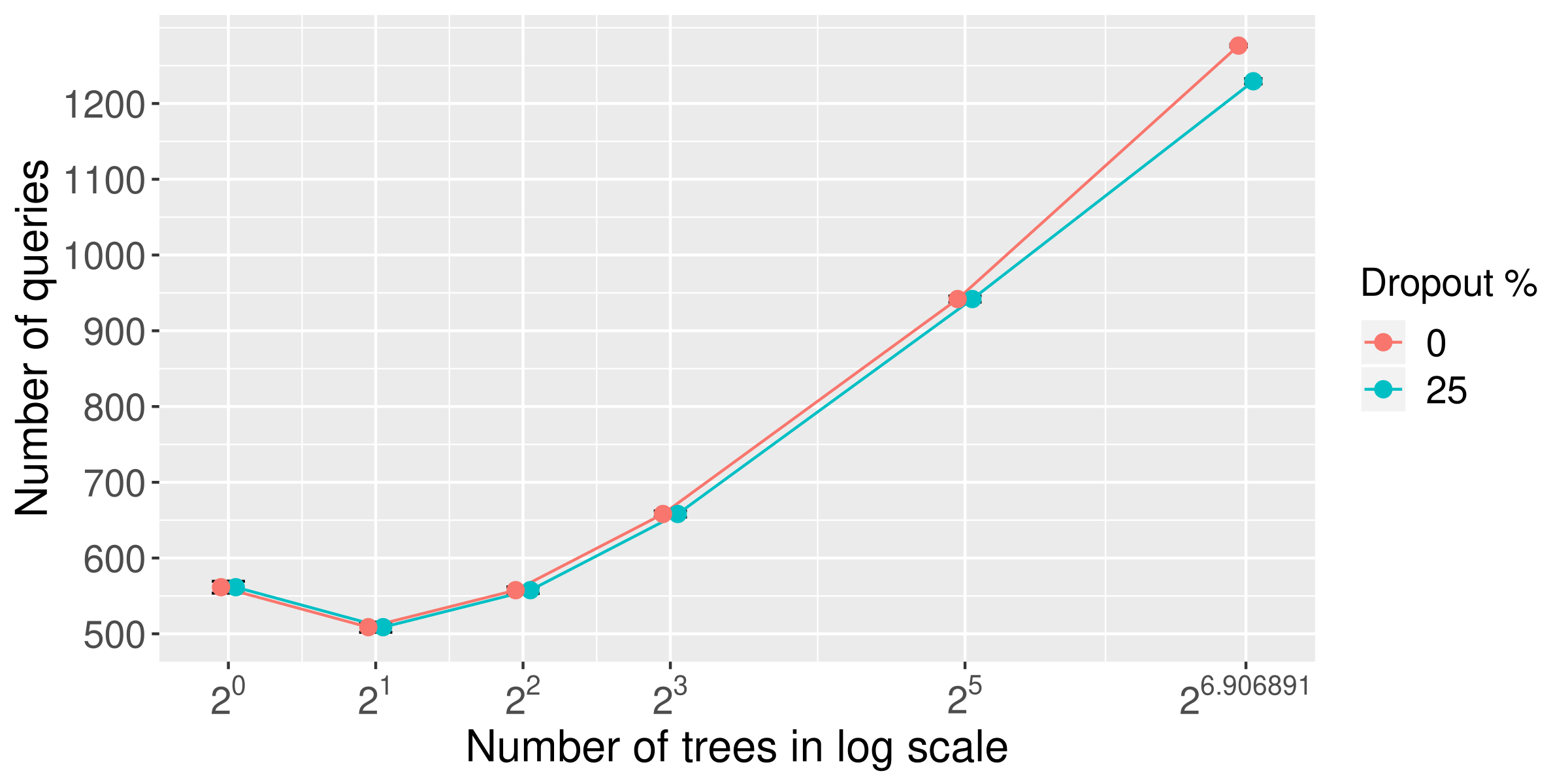}}
\caption{Additional performance metrics for C-FAR on our dataset. \textbf{Top.} adjusted mutual information (AMI) between the true and the cluster estimated by the algorithm vs number of input trees. An AMI of 1 means the estimated clusters are equal to the true; higher number means better performance. \textbf{Bottom} Number of queries (bottom) vs number of input trees in log scale. From the one-tree algorithm of \cite{gentile2019flattening} to just using 8 trees in C-FAR, adjusted mutual information increases sharply while the average number of queries only increase by 17\% (from 561 to 658)}
\label{fig:ami}
\end{center}
\vskip -0.2in
\end{figure}
As seen from Figures \ref{fig:cluster.runtime} and \ref{fig:ami}, the C-FAR algorithm has significantly better accuracy over the original one-tree algorithm of \cite{gentile2019flattening} while paying little in computation penalty. The average number of clusters drops from more than twice the true number with one tree to within 10\% of the true with 8 trees (cf. Figure \ref{fig:cluster.runtime}), while adjusted mutual information climbs from 0.7 to 0.94. In addition, the number of queries of C-FAR scales as $\log(m)$ (cf. Figure \ref{fig:ami}), with total runtime scales linearly in $m$ (cf. Figure \ref{fig:cluster.runtime}).

\subsection{Summary}\label{sec:summary}
This paper presents C-FAR, an ensemble clustering algorithm that takes in an arbitrary number of hierarchical clustering trees, prompts the user for pairwise comparisons between strategically chosen pairs of clusters, and outputs a flattened clustering. It generalizes the one-tree algorithm of \cite{gentile2019flattening}. By taking advantage of the common information across multiple trees, our algorithm performs significantly better than the one-tree case in simulated neuroscience data, with only a linear increase in total computation time. These results indicate that C-FAR is highly suitable as a model selection and assessment tool in clustering tasks, especially in the spike-sorting problem in neurosience. 

\bibliography{nsf-refs}
\bibliographystyle{icml2020}

\end{document}